\documentclass{article}

\usepackage{arxiv}

\usepackage[utf8]{inputenc} 
\usepackage[T1]{fontenc}    
\usepackage{hyperref}       
\usepackage{url}            
\usepackage{booktabs}       
\usepackage[sort&compress,numbers]{natbib}
\usepackage{amsfonts}       
\usepackage{nicefrac}       
\usepackage{microtype}      
\usepackage{lipsum}
\usepackage{graphicx}
\graphicspath{ {./images/} }

\usepackage{amsfonts}
\usepackage{amsmath}
\usepackage{amssymb}
\usepackage{amsthm}
\newtheorem{theorem}{Theorem}[section]

\newtheorem{definition}[theorem]{Definition}

\newcommand{\T}{^{\mathsf{T}}}

\usepackage{xcolor}
\usepackage{tcolorbox}

\title{
    Similarity of Neural Network Representations in Superposition
}

\author{
 Sunny Liu\textsuperscript{*}\\
  Cold Spring Harbor Laboratory\\
  \texttt{sliu@cshl.edu} \\
   \And
 Habon Issa\textsuperscript{*}\\
  Cold Spring Harbor Laboratory\\
  \texttt{issa@cshl.edu} \\
  \And
  André Longon\\
  UC San Diego\\
  \texttt{alongon@ucsd.edu}\\
  \And
  Liv Gorton\\
  Anthropic\\
  \And
  Meenakshi Khosla\\
  UC San Diego\\
  \And
  Alex Williams\\
  New York University\\
   Flatiron Institute\\
  \And
  David Klindt \\
  Cold Spring Harbor Laboratory\\
  \texttt{klindt@cshl.edu} \\
}

\begin{document}
\maketitle
{\renewcommand{\thefootnote}{}\footnotetext{\textsuperscript{*}These authors contributed equally to this work.}}
\begin{abstract}
Comparing internal representations is a central goal in neuroscience and machine learning, but standard linear alignment metrics (Representational Similarity Analysis, Centered Kernel Alignment, and linear regression) are frequently applied to neural activity coordinates rather than on the underlying features. We show this matters when neural systems operate in \textit{superposition}, encoding more features than they have neurons via linear compression. Closed-form derivations prove that these metrics depend on the Gram matrices of each system's projection, not on the latent features themselves: alignment thus combines \emph{what} a system represents with \emph{how} it is encoded. For those interested in what features two systems share, this is a problem: Two networks can have identical feature content yet appear more dissimilar than networks exhibiting partial feature overlap. This apparent misalignment need not reflect lost information as compressed sensing guarantees sparse features remain recoverable from the compressed activity. We confirm this by training supervised TopK sparse autoencoders that realize solvable compressed sensing by construction, finding alignment on recovered latents restored even when raw-activation alignment remains deflated. We extend the result to unsupervised SAEs trained without ground-truth latents, and to pretrained vision and language model SAEs, where SAE-latent alignment exceeds raw-activation alignment, consistent with superposition in real systems.
\end{abstract}


\section{Introduction}
\label{sec:intro}

\begin{figure}[tb]
    \centering
    \includegraphics[width=.8\textwidth]{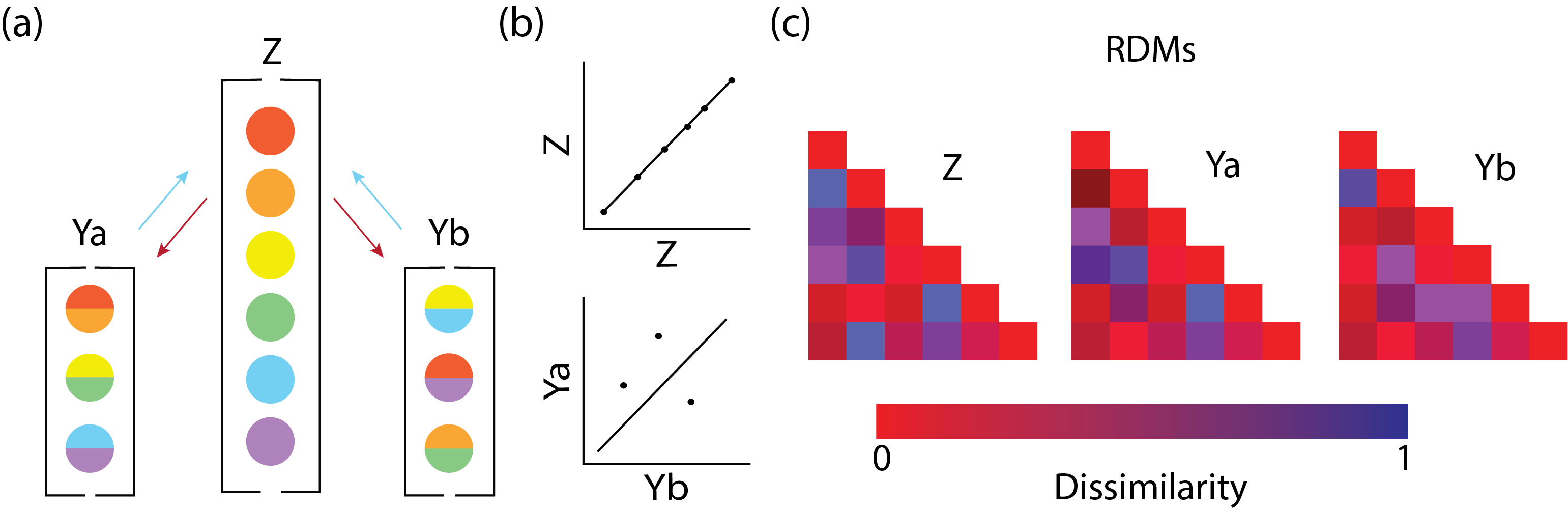}
    \caption{%
    \textbf{Illustration of core idea.}
    \textbf{a,} \textit{Superposition:} Two networks ($Y_a$, $Y_b$) share a common feature set ($Z$) but compress it differently (red arrows). Aligning the recovered latents (blue arrows) reveals identical feature content.
    \textbf{b,} \textit{Linear regression:} pairwise correlation between recovered latents reaches $1.0$, exceeding correlation between raw activations.
    \textbf{c,} \textit{Representational similarity analysis:} pairwise dissimilarity matrices of responses are correlated to produce an alignment score; as with regression, the score on recovered latents reaches $1.0$ and exceeds the score on raw activations.}
    \label{fig:intro}
\end{figure}

\paragraph{Superposition.}
\textit{Superposition} posits that neural networks linearly encode more features than they have neurons, distributing features across overlapping neural codes \citep{smolensky_tensor_1990, elhage2022toy, klindt2025superposition}. Formally, if $z \in \mathbb{R}^n$ denotes a vector of latent features and $y \in \mathbb{R}^m$ the neural response, a system in superposition computes $y = Az$ where $A \in \mathbb{R}^{m \times n}$ with $m < n$.
To illustrate, consider \hyperref[fig:intro]{Figure~\ref*{fig:intro}a}: a six-dimensional latent vector $z$ is linearly projected into three-dimensional neural responses for two systems, $y_a$ and $y_b$. Although both systems encode the same features, their projection matrices differ, producing distinct neural activity patterns.
Superposition typically gives rise to \textit{mixed selectivity}, or \textit{polysemanticity}, where individual neurons respond to multiple independent features \citep{rigotti2013importance, arora_linear_2018}, though some neurons may remain selective to a single feature, i.e., \textit{clean selectivity} or \textit{monosemanticity}.

Importantly, while this compression is lossy in general, compressed sensing theory guarantees that the original features can be perfectly recovered from the lower-dimensional neural activity, provided the features are sufficiently sparse and the projection satisfies certain conditions \citep{donoho_compressed_2006} (see Section~\ref{sec:cs_theory}). Thus, information is scrambled and distributed but not destroyed, meaning a downstream system can still access it \citep{adler2024complexity, hanni2024mathematical, garg2026many}.
In this work we ask: \textit{How does superposition affect standard neural alignment metrics?}

\paragraph{Representational alignment.}
A central goal in both neuroscience and machine learning is to quantitatively compare the internal representations of different neural systems \citep{kriegeskorte2008representational, kornblith2019similarity, sucholutsky2023getting}. Standard practice is to record neural responses from two systems to a shared set of stimuli, then apply an alignment metric to quantify their similarity. The most widely used metrics include Representational Similarity Analysis (RSA) \citep{kriegeskorte2008representational}, which correlates pairwise dissimilarity matrices; Centered Kernel Alignment (CKA) \citep{kornblith2019similarity}, which compares kernel matrices of neural responses; and linear regression, which measures how well one representation can be linearly predicted from the other. Nonlinear extensions of these metrics exist --- kernel CKA, topological RSA, and shape metrics with nonlinear feature maps --- and may speak to feature-content similarity in ways our analysis does not address; our scope is the predominant linear case. These metrics have yielded insights into shared structure between artificial and biological visual systems \citep{yamins2014performance, khaligh2014deep, cadena2019deep, khosla2021cortical, conwell2024large, prince2024contrastive} and between language models and the human language network \citep{schrimpf2021language}. 
Notably, standard linear alignment metrics are frequently applied to the raw neural activations (e.g., the distinct neural activity patterns $Y_a$ and $Y_b$ in Figure~\ref{fig:intro}). This means that they operate as composite measures of \textit{what} two systems represent ($Z$) and \textit{how} they arrange them across their neurons. Our argument is not that these metrics are flawed but that they answer a specific question --- linear similarity in activity coordinates --- which can differ substantially from similarity of the underlying features when systems share content but differ in encoding.

\paragraph{Alignment under superposition.}
We propose that superposition creates a critical distinction between two notions of alignment that are often combined. When two networks encode the same features but mix them differently across neurons (inevitable given different initializations or training runs), their activity coordinates will differ even though feature content is identical (Figure~\ref{fig:intro}, left). RSA and CKA operate on pairwise similarities of neural responses; the more features are packed into fewer neurons, the more these similarities are distorted, and the lower measured alignment becomes --- consistent with the empirical finding that smaller models achieve lower alignment \citep{elmoznino2024high}. This reduced alignment does not reflect a genuine loss of information: under compressed sensing the original features remain recoverable from the compressed activity \citep{donoho_compressed_2006}, and disentangling superposition can recover the alignment that raw activations obscure \citep{klindt2023interpretable, longon2025superposition}. Sparse autoencoders (SAEs) provide a practical tool for recovering this latent structure \citep{cunningham2023sparse, bricken2023towards}.

\paragraph{Contributions.}
\textbf{1) Analytic theory.} We derive closed-form expressions showing how superposition deflates feature-based alignment when applying RSA, CKA with linear kernel, and linear regression to the raw neural activations (Section~\ref{sec:theory}). In each case, alignment depends on the Gram matrices $G = A\T A$ of the projection matrices rather than on the latent features themselves, revealing the precise mechanism by which superposition obscures feature alignment. \\
\textbf{2) Simulation studies.} We test our theory in settings of increasing complexity. First, we use random Gaussian projections to generate different neural systems, testing their alignment as a function of compression. We then replace random Gaussian projections with TopK sparse autoencoders, training them end-to-end to reconstruct latents (Section~\ref{sec:sup_sae}), realizing solvable compressed sensing by construction and validating the theory with architectures capable of recovering latents (Fig.~\ref{fig:unsup_sae}). Finally, we demonstrate that TopK SAEs trained on neural activity recover latent features and restore feature alignment even without access to ground truth (Section~\ref{sec:unsup_sae}, Fig.~\ref{fig:unsup_sae}). \\
\textbf{3) Empirical validation on pretrained models.} By training SAEs on pretrained vision encoders (DINO, CLIP), we show that SAE-latent linear alignment between models differs systematically from raw neural activation alignment, in the direction predicted by our theory (Section~\ref{sec:real_sae}, Fig.~\ref{fig:real_sae}).

\section{Theory}
\label{sec:theory}

Let $z \in \mathbb{R}^n$ be a vector of \textit{latent variables} and $y \in \mathbb{R}^m$ a vector of neural \textit{representations}.

\begin{definition}[\textbf{Superposition}]
    A neural representation $y$ is in \textit{superposition} if it is
    \begin{enumerate}
        \item a \textit{linear} function of the latent variables $y = Az$, and
        \item a \textit{low-dimensional projection}, i.e., $A \in \mathbb{R}^{m \times n}$ with $m < n$.
    \end{enumerate}
\end{definition}

\paragraph{Assumptions.}
Throughout, we make two assumptions.
First, the neural representations are in superposition, described by a matrix $A \in \mathbb{R}^{m \times n}$: \(y = A z\).
Second, \textit{white covariance}: for a dataset of $d$ inputs, the latent vectors $z_1, \dots, z_d$ are treated as i.i.d.\ random variables with zero mean, $\mathbb{E}[z_i] = \mathbf{0}$, and identity covariance, $\mathbb{E}[z_i z_i\T] = I_n$.
The condition $m < n$ implies that the columns of $A$ cannot all be orthogonal, so features necessarily interfere with one another in the neural code.

\subsection{Compressed Sensing}\label{sec:cs_theory}
The projection from a high-dimensional latent space to a lower-dimensional neural space is generally lossy. However, compressed sensing guarantees that the original latent features can be perfectly recovered, provided two conditions are met \citep{donoho_compressed_2006, candes2006stable}:
\textit{Sparsity:} the latent variables are sparse, $\|z\|_0 \leq k$ for some $k \ll n$.
\textit{Restricted Isometry Property (RIP):} the projection $A$ approximately preserves distances between sparse vectors. Random Gaussian matrices satisfy RIP with high probability provided $m = \mathcal{O}(k \ln(n/k))$.
If these conditions are not fully satisfied, we incur an irreducible reconstruction error when recovering the latent features. This error lowers the ceiling of representational alignment, correctly reflecting the fact that if two features are separable in one system but cannot be separated in another system, this should count as a genuine misalignment.

\subsection{Representational Similarity Matrix (RSM)}
For a dataset of neural responses $Y = (y_1, \dots, y_d)$, the \textit{representational similarity matrix} is
\begin{equation}
    M(Y)_{ij} = \langle y_i, y_j \rangle, \qquad \forall i,j \in \{1, \dots, d\}.
\end{equation}
Under the linearity assumption,
\begin{equation}
    M(Y)_{ij} = \langle A z_i, A z_j \rangle = z_i\T A\T A z_j = z_i\T G z_j,
\end{equation}
where $G := A\T A \in \mathbb{R}^{n \times n}$ is the Gram matrix of the projection. The RSM thus measures similarity between latent variables under the inner product $\langle z_i, z_j \rangle_G := z_i\T G z_j$ induced by the neural code rather than the standard Euclidean inner product.

\section{Alignment Under Superposition}
\label{sec:alignment}

Consider two neural representations in superposition, with projection matrices $A_a \in \mathbb{R}^{m_a \times n}$, $A_b \in \mathbb{R}^{m_b \times n}$, generating responses $Y_a = A_a Z \in \mathbb{R}^{m_a \times d}$ and $Y_b = A_b Z \in \mathbb{R}^{m_b \times d}$ to the same latent dataset $Z = (z_1, \dots, z_d)$. This setup captures, for example, two differently initialized networks trained on the same task: they share the same underlying features but may produce distinct neural responses due to different projection matrices. We now analyze how standard alignment metrics behave in this scenario. The key insight is that any linear alignment metric applied to $Y_a$ and $Y_b$ will also reflect differences between $A_a$ and $A_b$, even when, from the perspective of compressed sensing, the two systems encode the same latent variables \(Z\).

\subsection{Representational Similarity Analysis (RSA)}
The RSA metric is the Pearson correlation between the vectorized upper-triangular elements of two RSMs.\footnote{Some implementations use Spearman rank correlation; since Spearman is Pearson on ranked vectors, the asymptotic scaling we derive holds qualitatively for both.}
Denoting these vectors $r_a, r_b \in \mathbb{R}^{d(d-1)/2}$,
\begin{equation}
    \rho(Y_a, Y_b) = \frac{\text{Cov}(r_a, r_b)}{\sqrt{\text{Var}(r_a)\text{Var}(r_b)}}.
\end{equation}

\begin{theorem}[Asymptotic RSA Alignment]\label{thm:rsa-correlation}
The RSA correlation between two representations $Y_a$ and $Y_b$ in superposition is the cosine similarity between their respective Gram matrices, $G_a = A_a\T A_a$ and $G_b = A_b\T A_b$:
\begin{equation}
    \rho(Y_a, Y_b) \approx \frac{\operatorname{Tr}(G_a G_b)}{\sqrt{\operatorname{Tr}(G_a^2) \operatorname{Tr}(G_b^2)}} = \frac{\langle G_a, G_b \rangle_F}{\|G_a\|_F \|G_b\|_F}.
\end{equation}
\end{theorem}
\textit{Proof in Appendix~\ref{app:proof_rsa}.} RSA alignment depends entirely on the Gram matrices of the projection, not on the latent features themselves. Two systems encoding identical features will show lower alignment scores whenever their projections induce different geometries in neural space.

\paragraph{Why alignment decreases with compression}
\label{para:why_decrease}
Theorem~\ref{thm:rsa-correlation} expresses alignment as a Gram-matrix similarity but does not on its own predict a value. Under \emph{independent} random projections --- the case where two systems share feature content but encode with unrelated geometries --- we obtain a quantitative prediction. (Networks sharing inductive biases may have correlated projections and lie outside this analysis.) We compute the expected numerator $\mathbb{E}[\operatorname{Tr}(G_a G_b)]$ when $A_a$ and $A_b$ are drawn independently with i.i.d.\ entries of mean zero\footnote{An arbitrary mean makes the expression more complicated but yields the same scaling with $m$.} and variance $\sigma^2$. By the cyclic property of the trace, $\operatorname{Tr}(G_a G_b) = \|A_b A_a\T\|_F^2$. Writing $C = A_b A_a\T \in \mathbb{R}^{m \times m}$, independence gives $\mathbb{E}[C_{\ell i}^2] = n\sigma^4$, and summing over all $m^2$ entries yields $\mathbb{E}[\operatorname{Tr}(G_a G_b)] = m^2 n \sigma^4$. For the denominator,
\[
\mathbb{E}[\operatorname{Tr}(G^2)] = mn\bigl[(m+n-2)\sigma^4 + \mu_4\bigr],
\]
where $\mu_4 = \mathbb{E}[A_{ij}^4]$. The ratio of expectations approximates the expected alignment:
\[
\frac{\mathbb{E}[\operatorname{Tr}(G_a G_b)]}{\mathbb{E}[\operatorname{Tr}(G^2)]} = \frac{m\,\sigma^4}{(m+n-2)\sigma^4 + \mu_4} \;\approx\; \frac{m}{n} \quad \text{for } m \ll n,
\]
where the approximation holds for any distribution with finite 4th moment. Alignment vanishes as $m/n \to 0$: 2 random $m$-dimensional subspaces have less room to overlap as $m$ shrinks relative to $n$.

\subsection{Centered Kernel Alignment (CKA)}
The CKA metric with linear kernel $K = Y\T Y$ satisfies (Appendix~\ref{app:proof_cka_lin}):
\begin{theorem}[Asymptotic Linear CKA Alignment]\label{thm:cka_lin}
The CKA alignment with linear kernel between $Y_a$ and $Y_b$ in superposition is the cosine similarity between their Gram matrices:
\begin{equation}
    \text{CKA}_{\mathrm{Lin}}(Y_a, Y_b) \approx \frac{\operatorname{Tr}(G_a G_b)}{\sqrt{\operatorname{Tr}(G_a^2) \operatorname{Tr}(G_b^2)}},
\end{equation}
equivalent to the asymptotic RSA result (Theorem~\ref{thm:rsa-correlation}).
\end{theorem}

\subsection{Linear Regression}
Predicting $Y_b$ from $Y_a$ via $\hat{Y}_b = W Y_a + E$ yields, in the asymptotic limit (Appendix~\ref{app:proof_linreg}):

\begin{theorem}[Asymptotic Linear Regression]\label{thm:ols}
\hspace{-0.5em}
\begin{enumerate}
\item \textbf{Optimal weights:} $\hat{W} \approx A_b A_a\T (A_a A_a\T)^{-1}$.
\item \textbf{MSE:} $\text{MSE}(Y_b | Y_a) \approx \frac{1}{m_b} \| A_b - \hat{W} A_a \|_F^2$.
\item \textbf{Explained variance:} $R^2 = 1 - \frac{\operatorname{Tr}((A_b - \hat{W} A_a)\T (A_b - \hat{W} A_a))}{\operatorname{Tr}(A_b\T A_b)}$.
\item \textbf{Pearson correlation:} $\rho(\hat Y_b, Y_b)_{ij} = \frac{(\hat W A_a A_b\T)_{ij}}{\sqrt{(\hat W A_a A_b\T)_{ii} (A_b A_b\T)_{jj}}}$.
\end{enumerate}
\end{theorem}

\subsection{Feature Overlap}
\label{sec:feature_overlap}
So far we have considered the idealized case in which two systems encode exactly the same set of features. We now turn to the more realistic scenario: two systems sharing only a subset of their features. Let $l_a$, $l_b$ denote the number of features in systems $A$ and $B$, and $l_{ab}$ the number shared. The full latent space is $n = l_a + l_b - l_{ab}$. Ordering features so that the first $l_a$ dimensions of $z$ belong to $A$, the last $l_b$ to $B$, and the overlap occupies positions $(l_a - l_{ab} + 1, \dots, l_a)$, the projection matrices have block structure making $G_a, G_b$ block-diagonal with supports on the $A$- and $B$-blocks respectively. Substituting into the RSA/CKA formula:
\begin{equation*}
  \frac{\operatorname{Tr}(G_a G_b)}{\sqrt{\operatorname{Tr}(G_a^2)\,\operatorname{Tr}(G_b^2)}} = \frac{\sum_{i=l_a - l_{ab}+1}^{l_a} (G_a G_b)_{ii}}{\sqrt{\|\tilde G_a\|_F^2 \|\tilde G_b\|_F^2}}.
\end{equation*}
This expression depends on both feature overlap (non-zero terms in the numerator) and the degree of superposition compression (which shapes the Gram matrices). The interplay is not straightforward to disentangle analytically; we investigate it through simulation below.
\begin{figure}[t]
    \centering
    \includegraphics[width=0.99\textwidth]{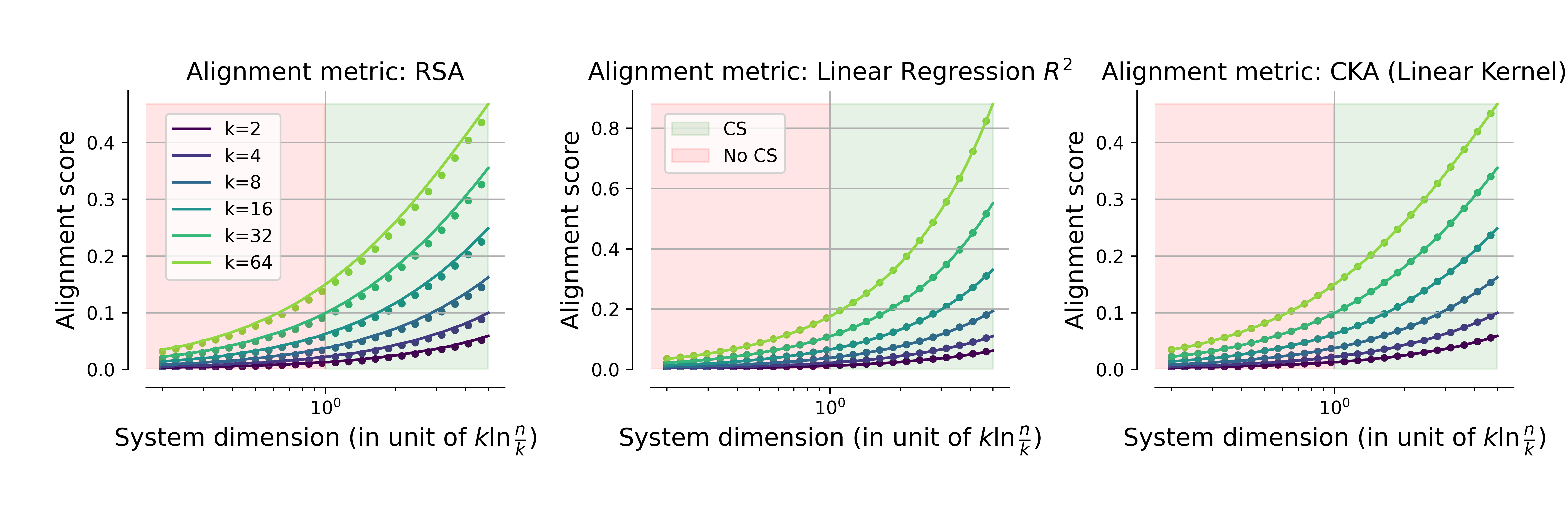}
    \vspace{-20pt}
    \caption{\textbf{Gaussian-projection validation (full overlap).} Alignment measured by RSA (left), linear-regression $R^2$ (middle), and linear CKA (right), as a function of system dimension $m$ in units of $k\ln(n/k)$, across multiple sparsity levels $k$. Analytical predictions (solid curves) match simulation (dots). The No-CS regime (red shading, $m < k\ln(n/k)$) is where exact latent recovery is not guaranteed; the CS regime (green shading) is where it is.}
    \label{fig:gauss_full}
\end{figure}

\section{Superposition in Random Gaussian Projections}
\label{sec:gau_sim}

Inspired by standard results in compressed sensing \cite{donoho_compressed_2006}, we validate our theory by simulating superposition through randomly sampling Gaussian projection matrices from a shared set of pre-determined latents $Z \in \mathbb{R}^{n \times d}$.  We simulate two scenarios: the first where two systems share completely overlapping set of features, and the second where two systems share partially overlapping sets of features. 

\subsection{Full feature overlap}
\label{subsec: gau_full_overlap}
\paragraph{Setup.}
For two systems A and B, their respective neural responses $Y_a, Y_b \in \mathbb{R}^{m \times d}$ are obtained by projecting the latents $Z$ through matrices $A_a, A_b \in \mathbb{R}^{m \times n}$, the entries of which are drawn from a standard normal distribution: 
\begin{equation}
    Y_a = A_a Z \,\,, \,\, Y_b = A_b Z \,, \qquad  A_{ij} \sim \mathcal{N}(0,1)
\end{equation}
Details in \ref{app:gaussian_sims}.
We simulate systems across a range of $m$ and sparsity $k$ and measure their alignment.

\paragraph{Results.}
Across all three metrics, closed-form predictions (solid curves) closely match simulation (dots; Figure~\ref{fig:gauss_full}). Alignment scales monotonically with $m$, consistent with the $m/(m+n-2)$ ratio of expectations derived in Section~\ref{para:why_decrease}, which reduces to the $m/n$ approximation for $m \ll n$. The three metrics track one another tightly, as expected from Theorems~\ref{thm:rsa-correlation} and~\ref{thm:cka_lin}, which give RSA and linear CKA the same asymptotic limit; linear regression $R^2$ behaves analogously. That two systems with identical feature content can register near-zero raw-activation alignment is the central empirical takeaway here --- deflation is predictable from encoding geometry alone, with no information lost.

\subsection{Partial feature overlap}
\label{subsec: gau_part_overlap}
\paragraph{Setup.}
We extend Section~\ref{subsec: gau_full_overlap} to two systems sharing only a fraction of their features. Let $l_a, l_b$ be the per-system feature counts and $l_{ab}$ the shared count; the full latent dimension is $n = l_a + l_b - l_{ab}$. The \textit{feature overlap ratio} $u \equiv l_{ab}/\sqrt{l_a l_b}$ ranges from $0$ (disjoint) to $1$ (full overlap, recovering Section~\ref{subsec: gau_full_overlap}). Setting $l_a = l_b = l$, projections $A_a, A_b$ are column-masked so each system encodes only its own features (Appendix~\ref{app:gaussian_sims}). We fix $m_a$ in the CS regime and sweep $m_b$ under different $u$, using CKA-Lin as a representative metric per Theorems~\ref{thm:rsa-correlation}--\ref{thm:cka_lin}.

\begin{figure}[t]
    \centering
    \includegraphics[width=1\textwidth]{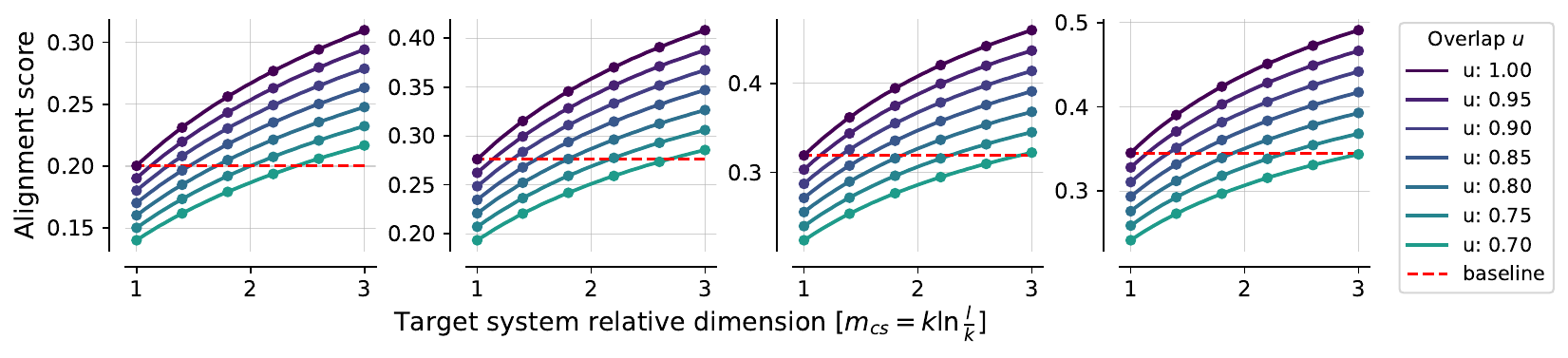}
    \vspace*{-0.7cm}
    \caption{\textbf{Gaussian-projection validation (partial feature overlap).} CKA (linear kernel) alignment as a function of target system dimension $m_b$ (in units of $m_{\text{cs}} = k\ln(l/k)$), with $m_a$ fixed at $3 m_{\text{cs}}$ in the CS regime. Each curve corresponds to a different feature overlap ratio $u$; the red dashed line marks the alignment of a fully overlapping pair ($u=1$) at $m_b = m_{\text{cs}}$. Curves above this line demonstrate \emph{rank inversion}: partial-overlap pairs ($u<1$) at lower compression score higher than fully overlapping pairs at heavier compression. Subplots left to right increasing sparsity fraction $k/l$.}
    \label{fig:gauss_partial_diff_dims}
\end{figure}

\paragraph{Results.}
Figure~\ref{fig:gauss_partial_diff_dims} reveals a striking \emph{rank inversion}. The red dashed line marks the alignment achieved by a fully overlapping pair ($u=1$) at $m_b = k\ln(l/k)$; curves above this line correspond to partial-overlap pairs ($u<1$) at less compressed $m_b$. That such curves exist means two systems sharing strictly less feature content can register \emph{higher} raw-activation alignment than two systems encoding identical features, purely because the latter pair is more heavily compressed. This is the cleanest illustration of the central thesis on simulated data: when alignment metrics conflate \emph{what} and \emph{how}, the ordering they induce need not reflect the ordering of feature overlap.

\section{Learnable Superposition via Supervised TopK SAEs}
\label{sec:sup_sae}

Our closed-form theory is derived for arbitrary projection matrices $A$; the Gaussian-random case (Appendix~\ref{app:gaussian_sims}) validates the theory but raises two concerns: (i) Gaussian projections are a worst case for alignment because $A_a$ and $A_b$ are statistically independent, and (ii) while compressed-sensing theory \emph{guarantees} that sparse latents are recoverable when $m \gtrsim k\ln(n/k)$, the Gaussian setup does not exhibit a \emph{practical} decoder that achieves this \citep{oneil2024compute}. We therefore introduce a learnable variant in which superposition is realized by a trained TopK sparse autoencoder.

\paragraph{Setup.}
A \textit{supervised TopK SAE} is a linear encoder--decoder pair trained end-to-end to reconstruct sparse latent features:
\begin{equation}
    \hat{z} = \text{TopK}_k(E \, D z), \qquad D \in \mathbb{R}^{m \times n},\ E \in \mathbb{R}^{n \times m},
\end{equation}
with reconstruction loss $\mathcal{L} = \mathbb{E}[\|z - \hat z\|_2^2]$ and gradients through TopK passed via the straight-through estimator. TopK is applied at the $n$-dimensional output to exploit the known sparsity $k$. The trained decoder $D$ plays the role of $A$ in Section~\ref{sec:alignment}: the activations $Y = D Z$ are the raw neural responses on which alignment is measured. Unlike Gaussian $A$, here $D$ is \emph{learned} such that a matching encoder $E$ exists --- compressed sensing is realized by construction. We train two such SAEs, $(D_a, E_a)$ and $(D_b, E_b)$, independently from different random seeds on identical latent data. Because the objective is reconstruction (not alignment), the two decoders converge to distinct projection geometries while both admit accurate recovery via their paired encoders --- exactly the regime of interest: identical feature content, different encodings. Full hyperparameters in Appendix~\ref{app:sup_sae_details}.

\paragraph{Measurements.}
For each trained pair and a sweep over $m$ we compute two alignment scores. \emph{Raw alignment} $\text{metric}(Y_a, Y_b)$ with $Y_a = D_a Z$, $Y_b = D_b Z$, which Section~\ref{sec:alignment} predicts to be deflated by superposition; and \emph{reconstructed-latent alignment} $\text{metric}(\hat Z_a, Y_b)$ with $\hat Z_a = \text{TopK}_k(E_a D_a Z)$, which we predict to approach $1.0$ in the CS regime since $\hat Z_a \to Z$ and $Y_b = D_b Z$ is a linear function of $Z$. We use linear regression $R^2$ as the alignment metric.

\paragraph{Results.}
We plot the supervised-SAE results as the green ceiling curve in Figure~\ref{fig:unsup_sae}, alongside the unsupervised setup of Section~\ref{sec:unsup_sae}. Raw-activation alignment (blue) decreases with compression, consistent with Section~\ref{sec:alignment}: two independently trained encoders converge to distinct Gram matrices despite encoding identical feature content. Reconstructed-latent alignment, by contrast, remains near $1.0$ throughout the CS regime and degrades only when $m$ falls below the recoverability threshold. This is the cleanest demonstration of the central claim: the apparent misalignment under superposition reflects encoding choice, not information content.

\section{Recovering Latents with Unsupervised SAEs}
\label{sec:unsup_sae}

\begin{figure}[t]
    \centering
    \includegraphics[width=1\textwidth]{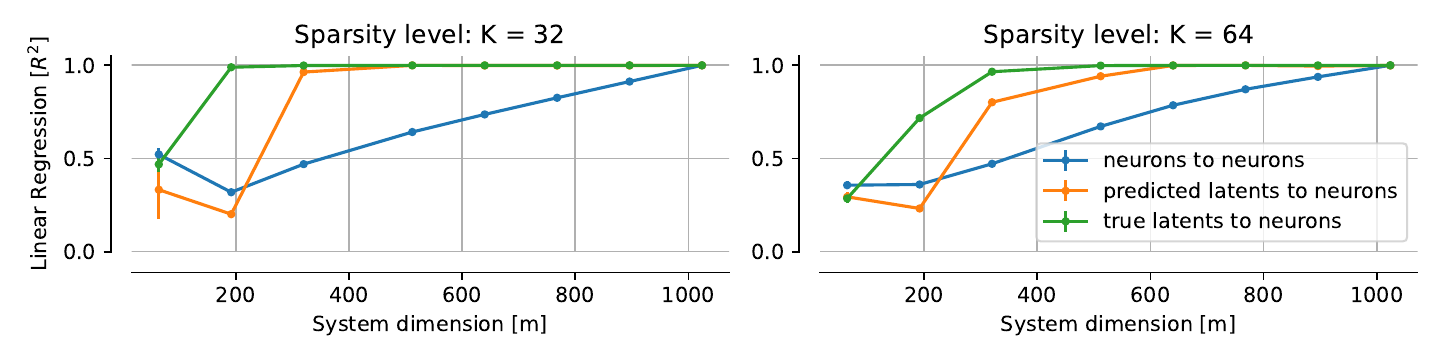}
    \caption{\textbf{Unsupervised SAEs partially recover the alignment deflated by superposition.} Without access to ground-truth latents, an unsupervised TopK SAE trained on neural activity $Y$ recovers latents whose pairwise alignment (orange) is substantially higher than the raw-activation baseline (blue), approaching the supervised-SAE ceiling (green) in the CS regime. The gap between green and orange quantifies the cost of unsupervised identification.}
    \label{fig:unsup_sae}
\end{figure}

The supervised setup of Section~\ref{sec:sup_sae} relies on ground-truth latents at training time, which are unavailable for real neural systems. We now ask the practical question: can an \textit{unsupervised} SAE, trained only on neural activity $Y$ without seeing $Z$, recover the latents well enough to restore alignment?

\paragraph{Setup.}
Given neural responses $Y \in \mathbb{R}^{m \times d}$ from a single system (Gaussian-projected or supervised-SAE-encoded), we train a TopK SAE
\begin{equation}
    \hat y = D_u \cdot \text{TopK}_k(E_u \, y), \qquad E_u \in \mathbb{R}^{n' \times m},\ D_u \in \mathbb{R}^{m \times n'},
\end{equation}
with $n' \geq n$ (we set $n' = n$ for direct comparison) and loss $\mathcal{L} = \mathbb{E}[\|y - \hat y\|^2]$. The sparse activations $\hat Z = \text{TopK}_k(E_u y) \in \mathbb{R}^{n'}$ are taken as recovered latents. For a pair of systems we train two SAEs independently on $Y_a$ and $Y_b$, yielding $\hat Z_a$ and $\hat Z_b$, and measure alignment via linear regression $R^2$ (which is invariant to the arbitrary index permutations induced by independent dictionary learning).

\paragraph{Identifiability caveat.}
This is non-trivial: SAEs trained on identical data can learn different feature decompositions \citep{paulo2025sparse, leask2025sparse}, and feature splitting and absorption complicate direct comparison \citep{chanin2024absorption, fel2025archetypal}. Our setup is well-specified by construction --- ground-truth $Z$ and projection $A$ are both known --- allowing us to separately quantify (a) latent recovery in the CS regime and (b) the fraction of the alignment deflation an unsupervised SAE actually undoes.

\paragraph{Results.}
Figure~\ref{fig:unsup_sae} compares three curves at fixed sparsity. The raw-activation baseline (blue) decreases with compression, reproducing Section~\ref{sec:gau_sim}. The supervised ceiling (green) remains near $1.0$ throughout the CS regime: with ground-truth latents, alignment is perfectly recovered. The unsupervised curve (orange) sits between the two, substantially above the raw baseline and approaching the supervised ceiling within the CS regime. The gap between green and orange quantifies the cost of identifying the latent basis from data alone; that this cost is small in the CS regime is the central practical finding of this section. Outside the CS regime all three curves degrade, consistent with compressed sensing no longer guaranteeing recovery.

\section{Empirical Validation on Pretrained Models}
\label{sec:real_sae}

We now turn to the setting of greatest practical interest: pretrained networks whose true latent structure is unknown. The question is whether the alignment gap our theory predicts --- higher alignment in SAE-latent space than in raw-activation space --- holds for real models.

\paragraph{Setup.}
We compare pairs of pretrained vision encoders (DINO and CLIP) at corresponding layers, using the SAEs documented in the ViT Prisma repository \citep{joseph2025prismaopensourcetoolkit}. Activations are extracted from residual-post and MLP-out sublayers, restricted to the CLS token, and evaluated on the ImageNet1K validation set \citep{ILSVRC15}. For each pair $(a, b)$ we train TopK SAEs jointly so that the two systems share a sparse latent space, with loss
\begin{equation}
    \mathcal{L} = \sum_{i,j \in \{a,b\}} \mathrm{MSE}(y_i, \hat{y}_{ij}), \qquad \hat{y}_{ij} = D_i \cdot \text{TopK}_k(E_j \, y_j),
\end{equation}
where $E_j$ encodes system $j$'s activations into the shared sparse code and $D_i$ decodes back into system $i$'s activation space. We use latent dimension $n = 2m$. Full details in Appendix~\ref{app:real_sae_details}.

\paragraph{Two-stage validation.}
The theory of Section~\ref{sec:alignment} is derived under strict superposition ($y = Az$ with $z$ sparse), an assumption real activations may violate. We therefore evaluate alignment in two settings. The \emph{raw} setting uses the actual hidden states $y$, which may or may not be in superposition. The \emph{reconstructed} setting uses the SAE reconstructions $\hat y$, which are in superposition by construction (a sparse latent linearly decoded into activation space). The theory must hold in the reconstructed setting; whether it also holds in the raw setting is itself an empirical test of how well the superposition assumption describes real activations.

\paragraph{Measurements.}
For each system pair we compute \emph{raw alignment} $\mathrm{metric}(Y_a, Y_b)$ and \emph{latent alignment} $\mathrm{metric}(Z_a, Y_b)$, where $Z_a = \text{TopK}_k(E_a Y_a)$ is system $a$'s sparse code. As alignment metric we use Pearson $r$ between OLS predictions and targets (i.e., linear-regression alignment).

\begin{figure}[t]
    \centering
    \includegraphics[width=0.7\textwidth]{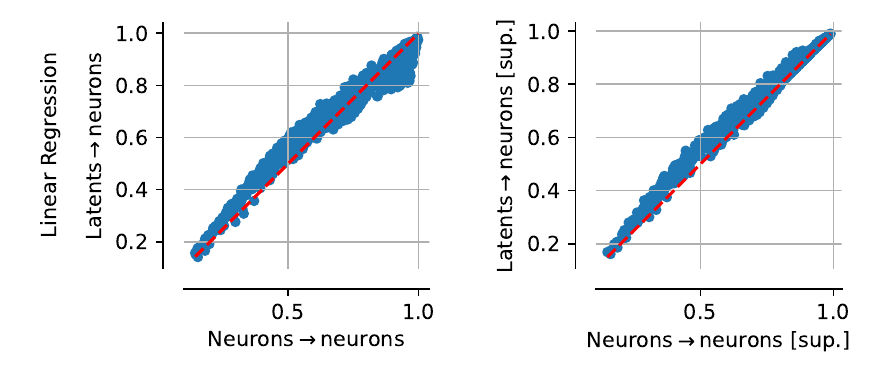}
    \vspace{-10pt}
    \caption{\textbf{SAEs are more aligned on pretrained vision models.} Each scatter point corresponds to one (model pair, layer, sublayer) triple. X-axis: linear-regression alignment between two systems' raw activations. Y-axis: linear-regression alignment from one system's SAE latents to the other system's activations. \textbf{Left:} real activations as targets. \textbf{Right:} SAE-reconstructed activations as targets, which are in superposition by construction (positive control for the theory). Points above the red dashed diagonal indicate higher latent-to-raw than raw-to-raw alignment; their predominance in both panels is consistent with superposition operating in pretrained models.}
    \vspace{-10pt}
    \label{fig:real_sae}
\end{figure}

\paragraph{Results.} 
For the majority of system pairs and layers tested, latent alignment exceeds raw alignment: scatter points lie predominantly on or above the diagonal of equality (Figure~\ref{fig:real_sae}). The right panel serves as a positive control --- the reconstructed activations are in superposition by construction, so the theory must apply, and the pattern confirms it does. The left panel shows the same qualitative effect on raw activations, consistent with real activations being well-described by the superposition model. Across diverse architectures and layers, this implies that pretrained networks arrange features idiosyncratically across their neurons in ways that depress raw-activation alignment, and that SAE disentanglement partially restores feature-level alignment.

\section{Limitations}

Our analysis rests on two assumptions, both of which are testable rather than provable in general. First, neural systems operate in superposition. This is motivated by growing evidence of mixed selectivity in artificial \citep{elhage2022toy, bricken2023towards, klindt2023interpretable} and biological \citep{rigotti2013importance, klindt2023interpretable} systems, and is consistent with the experiments in Section~\ref{sec:real_sae}. Second, two different systems are likely to have dissimilar superposition projections; without this, raw-activation alignment would not be deflated. Visual neuroscience offers a partial counterargument: feature tuning is strikingly conserved across individuals and species, from V1 complex cells \citep{hubel1962receptive, hubel1968receptive} to face-selective cells \citep{quiroga2005invariant, kanwisher1997fusiform}, suggesting that some prominent features are reliably assigned dedicated neurons \citep{elhage2022toy, cadieu2014deep}.

The theoretical results further assume whitened i.i.d.\ latents. Real features co-occur, inducing correlations between projection matrices that our analysis does not capture; we expect qualitative conclusions to survive, but quantifying the effect of structured latent statistics is future work. Finally, the empirical evaluation is restricted to artificial vision encoders. Whether the predicted alignment gap transfers to biological recordings depends both on real activity being well-described by sparse linear superposition and on the quality of SAEs available for that data --- both open.

\section{Discussion}
\label{sec:discussion}

We have shown, both analytically and empirically, that superposition systematically lowers linear alignment scores measured on raw neural activations. The mechanism is direct: RSA, linear CKA, and linear regression applied to activations all depend on the Gram matrices of the compression, not on the latent features themselves. Random Gaussian projections (Appendix~\ref{app:gaussian_sims}) confirm the prediction; supervised TopK SAEs extend it to a setting where compressed sensing is realized by construction; unsupervised SAEs recover most of the deflated alignment without ground truth; and SAEs on pretrained models suggest the effect is present in real systems.

These results do not imply that linear alignment metrics are flawed. They answer a specific question --- linear similarity in activity coordinates --- which is a different question from similarity of the underlying features. By isolating the latter, we find that the answers can come apart, and that the gap is large enough to flip the relative ordering of system pairs (Section~\ref{subsec: gau_part_overlap}).

This observation also offers an interpretation of the empirical finding that alignment increases with model size \citep{huh2024platonic, kornblith2019similarity}: larger models represent features with less compression, reducing the distortion superposition introduces into pairwise similarity structure. Gr\"oger et al. \citep{groger2026revisitplatonic} have questioned the robustness of this trend for global metrics and found local neighborhood structure better preserved --- consistent with the geometry of linear projections, which distort global distances while preserving local neighborhoods. A formal analysis of superposition effects on local alignment metrics is an important direction for future work.

\paragraph{Looking forward.}
The prescription is concrete: to understand the extent of feature overlap between neural networks, alignment should be measured in latent feature space. Practical implementations include sparse autoencoders \citep{cunningham2023sparse, bricken2023towards}, dictionary learning \citep{olshausen1997sparse}, and identifiable representation learning \citep{reizinger2024cross, reizinger2025empirically}; developing and validating such pipelines is a critical next step.

\section{Acknowledgments}

We acknowledge Cold Spring Harbor HPC GPU cluster supported by grant S10OD028632-01.

\bibliography{references}

@article{sucholutsky2023getting,
  title={Getting aligned on representational alignment},
  author={Sucholutsky, Ilia and Muttenthaler, Lukas and Weller, Adrian and Peng, Andi and Bobu, Andreea and Kim, Been and Love, Bradley C and Grant, Erin and Groen, Iris and Achterberg, Jascha and others},
  journal={arXiv preprint arXiv:2310.13018},
  year={2023}
}

@article{kriegeskorte2008representational,
  title={Representational similarity analysis-connecting the branches of systems neuroscience},
  author={Kriegeskorte, Nikolaus and Mur, Marieke and Bandettini, Peter A},
  journal={Frontiers in systems neuroscience},
  volume={2},
  pages={249},
  year={2008},
  publisher={Frontiers}
}

@article{elhage2022toy,
  title={Toy models of superposition},
  author={Elhage, Nelson and Hume, Tristan and Olsson, Catherine and Schiefer, Nicholas and Henighan, Tom and Kravec, Shauna and Hatfield-Dodds, Zac and Lasenby, Robert and Drain, Dawn and Chen, Carol and others},
  journal={arXiv preprint arXiv:2209.10652},
  year={2022}
}

@article{elmoznino2024high,
  title={High-performing neural network models of visual cortex benefit from high latent dimensionality},
  author={Elmoznino, Eric and Bonner, Michael F},
  journal={PLoS computational biology},
  volume={20},
  number={1},
  pages={e101a1792},
  year={2024},
  publisher={Public Library of Science San Francisco, CA USA}
}

@article{olshausen1997sparse,
  title={Sparse coding with an overcomplete basis set: A strategy employed by V1?},
  author={Olshausen, Bruno A and Field, David J},
  journal={Vision research},
  volume={37},
  number={23},
  pages={3311--3325},
  year={1997},
  publisher={Elsevier}
}

@article{cunningham2023sparse,
  title={Sparse autoencoders find highly interpretable features in language models},
  author={Cunningham, Hoagy and Ewart, Aidan and Riggs, Logan and Huben, Robert and Sharkey, Lee},
  journal={arXiv preprint arXiv:2309.08600},
  year={2023}
}

@article{bricken2023towards,
  title={Towards Monosemanticity: Decomposing Language Models With Dictionary Learning},
  author={Bricken, Trenton and Chen, Andy and Anthropic, et al.},
  journal={Transformer Circuits},
  year={2023},
  url={https://transformer-circuits.pub/2023/monosemantic-features}
}

@article{candes2006stable,
  title={Stable signal recovery from incomplete and inaccurate measurements},
  author={Candes, Emmanuel J and Romberg, Justin K and Tao, Terence},
  journal={Communications on Pure and Applied Mathematics: A Journal Issued by the Courant Institute of Mathematical Sciences},
  volume={59},
  number={8},
  pages={1207--1223},
  year={2006},
  publisher={Wiley Online Library}
}

@article{huh2024platonic,
  title={The platonic representation hypothesis},
  author={Huh, Minyoung and Cheung, Brian and Wang, Tongzhou and Isola, Phillip},
  journal={arXiv preprint arXiv:2405.07987},
  year={2024}
}

@article{groger2026revisitplatonic,
  title={Revisiting the Platonic Representation Hypothesis: An Aristotelian View},
  author={Gr\"oger, Fabian and Wen, Sho and Brbi\'c, Maria},
  journal={arXiv preprint arXiv:2602.14486},
  year={2026}
}

@article{reizinger2024cross,
  title={Cross-entropy is all you need to invert the data generating process},
  author={Reizinger, Patrik and Bizeul, Alice and Juhos, Attila and Vogt, Julia E and Balestriero, Randall and Brendel, Wieland and Klindt, David},
  journal={arXiv preprint arXiv:2410.21869},
  year={2024}
}

@article{klindt2023interpretable,
      title={Identifying Interpretable Visual Features in Artificial and Biological Neural Systems}, 
      author={David Klindt and Sophia Sanborn and Francisco Acosta and Frédéric Poitevin and Nina Miolane},
      year={2023},
      eprint={2310.11431},
      archivePrefix={arXiv},
      primaryClass={stat.ML},
      url={https://arxiv.org/abs/2310.11431}, 
}

@article{oneil2024compute,
  title={Compute optimal inference and provable amortisation gap in sparse autoencoders},
  author={O'Neill, Charles and Gumran, Alim and Klindt, David},
  journal={arXiv preprint arXiv:2411.13117},
  year={2024}
}

@article{donoho_compressed_2006,
    title = {Compressed sensing},
    volume = {52},
    url = {https://ieeexplore.ieee.org/abstract/document/1614066/?casa_token=vtpGjU5mzFcAAAAA:rU2N5NCWY2K9IaaU0GHdJEuOj8P0dFk39KnF-rchFhrMrAe9T0XiWvCPGgJ5pszVR4-UWxvhvg},
    number = {4},
    urldate = {2024-06-05},
    journal = {IEEE Transactions on information theory},
    author = {Donoho, David L.},
    year = {2006},
    note = {Publisher: IEEE},
    pages = {1289--1306},
}

@article{paulo2025sparse,
  title={Sparse autoencoders trained on the same data learn different features},
  author={Paulo, Gon{\c{c}}alo and Belrose, Nora},
  journal={arXiv preprint arXiv:2501.16615},
  year={2025}
}

@inproceedings{leask2025sparse,
    title={Sparse Autoencoders Do Not Find Canonical Units of Analysis},
    author={Patrick Leask and Bart Bussmann and Michael T Pearce and Joseph Isaac Bloom and Curt Tigges and Noura Al Moubayed and Lee Sharkey and Neel Nanda},
    booktitle={The Thirteenth International Conference on Learning Representations},
    year={2025},
    url={https://openreview.net/forum?id=9ca9eHNrdH}
}

@article{chanin2024absorption,
  title={A is for absorption: Studying feature splitting and absorption in sparse autoencoders},
  author={Chanin, David and Wilken-Smith, James and Dulka, Tom{\'a}{\v{s}} and Bhatnagar, Hardik and Golechha, Satvik and Bloom, Joseph},
  journal={arXiv preprint arXiv:2409.14507},
  year={2024}
}

@inproceedings{fel2025archetypal,
    title={Archetypal {SAE}: Adaptive and Stable Dictionary Learning for Concept Extraction in Large Vision Models},
    author={Thomas Fel and Ekdeep Singh Lubana and Jacob S. Prince and Matthew Kowal and Victor Boutin and Isabel Papadimitriou and Binxu Wang and Martin Wattenberg and Demba E. Ba and Talia Konkle},
    booktitle={Forty-second International Conference on Machine Learning},
    year={2025},
    url={https://openreview.net/forum?id=9v1eW8HgMU}
}

@article{khaligh2014deep,
  title={Deep supervised, but not unsupervised, models may explain IT cortical representation},
  author={Khaligh-Razavi, Seyed-Mahdi and Kriegeskorte, Nikolaus},
  journal={PLoS computational biology},
  volume={10},
  number={11},
  pages={e1003915},
  year={2014},
  publisher={Public Library of Science San Francisco, USA}
}

@article{yamins2014performance,
  title={Performance-optimized hierarchical models predict neural responses in higher visual cortex},
  author={Yamins, Daniel LK and Hong, Ha and Cadieu, Charles F and Solomon, Ethan A and Seibert, Darren and DiCarlo, James J},
  journal={Proceedings of the national academy of sciences},
  volume={111},
  number={23},
  pages={8619--8624},
  year={2014},
  publisher={National Academy of Sciences}
}

@article{khosla2021cortical,
  title={Cortical response to naturalistic stimuli is largely predictable with deep neural networks},
  author={Khosla, Meenakshi and Ngo, Gia H and Jamison, Keith and Kuceyeski, Amy and Sabuncu, Mert R},
  journal={Science Advances},
  volume={7},
  number={22},
  pages={eabe7547},
  year={2021},
  publisher={American Association for the Advancement of Science}
}

@article{schrimpf2021language,
    author = {Martin Schrimpf  and Idan Asher Blank  and Greta Tuckute  and Carina Kauf  and Eghbal A. Hosseini  and Nancy Kanwisher  and Joshua B. Tenenbaum  and Evelina Fedorenko },
    title = {The neural architecture of language: Integrative modeling converges on predictive processing},
    journal = {Proceedings of the National Academy of Sciences},
    volume = {118},
    number = {45},
    pages = {e2105646118},
    year = {2021},
    doi = {10.1073/pnas.2105646118},
    URL = {https://www.pnas.org/doi/abs/10.1073/pnas.2105646118},
    eprint = {https://www.pnas.org/doi/pdf/10.1073/pnas.2105646118}
}

@article{conwell2024large,
  title={A large-scale examination of inductive biases shaping high-level visual representation in brains and machines},
  author={Conwell, Colin and Prince, Jacob S and Kay, Kendrick N and Alvarez, George A and Konkle, Talia},
  journal={Nature communications},
  volume={15},
  number={1},
  pages={9383},
  year={2024},
  publisher={Nature Publishing Group UK London}
}

@article{prince2024contrastive,
  title={Contrastive learning explains the emergence and function of visual category-selective regions},
  author={Prince, Jacob S and Alvarez, George A and Konkle, Talia},
  journal={Science Advances},
  volume={10},
  number={39},
  pages={eadl1776},
  year={2024},
  publisher={American Association for the Advancement of Science}
}

@article{cadena2019deep,
  title={Deep convolutional models improve predictions of macaque V1 responses to natural images},
  author={Cadena, Santiago A and Denfield, George H and Walker, Edgar Y and Gatys, Leon A and Tolias, Andreas S and Bethge, Matthias and Ecker, Alexander S},
  journal={PLoS computational biology},
  volume={15},
  number={4},
  pages={e1006897},
  year={2019},
  publisher={Public Library of Science San Francisco, CA USA}
}

@article{reizinger2025empirically,
  title={An Empirically Grounded Identifiability Theory Will Accelerate Self-Supervised Learning Research},
  author={Reizinger, Patrik and Balestriero, Randall and Klindt, David and Brendel, Wieland},
  journal={bioRxiv},
  year={2025}
}

@article{klindt2025superposition,
  title={From superposition to sparse codes: interpretable representations in neural networks},
  author={Klindt, David and O'Neill, Charles and Reizinger, Patrik and Maurer, Harald and Miolane, Nina},
  journal={arXiv preprint arXiv:2503.01824},
  year={2025}
}

@article{smolensky_tensor_1990,
	title        = {Tensor product variable binding and the representation of symbolic structures in connectionist systems},
	author       = {Smolensky, Paul},
	year         = 1990,
	journal      = {Artificial intelligence},
	volume       = 46,
	number       = {1-2},
	pages        = {159--216},
	url          = {https://www.sciencedirect.com/science/article/pii/000437029090007M},
	urldate      = {2024-05-29},
	note         = {Publisher: Elsevier}
}

@article{arora_linear_2018,
	title        = {Linear {Algebraic} {Structure} of {Word} {Senses}, with {Applications} to {Polysemy}},
	author       = {Arora, Sanjeev and Li, Yuanzhi and Liang, Yingyu and Ma, Tengyu and Risteski, Andrej},
	year         = 2018,
	month        = {dec},
	journal      = {Transactions of the Association for Computational Linguistics},
	volume       = 6,
	pages        = {483--495},
	doi          = {10.1162/tacl_a_00034},
	issn         = {2307-387X},
	url          = {https://direct.mit.edu/tacl/article/43451},
	urldate      = {2024-05-28},
	language     = {en}
}

@InProceedings{kornblith2019similarity,
    title        = {Similarity of Neural Network Representations Revisited},
    author       = {Kornblith, Simon and Norouzi, Mohammad and Lee, Honglak and Hinton, Geoffrey},
    booktitle    = {Proceedings of the 36th International Conference on Machine Learning},
    pages        = {3519--3529},
    year         = {2019},
    editor       = {Chaudhuri, Kamalika and Salakhutdinov, Ruslan},
    volume       = {97},
    series       = {Proceedings of Machine Learning Research},
    month        = {09--15 Jun},
    publisher    = {PMLR},
    url          = {https://proceedings.mlr.press/v97/kornblith19a.html},
}

@article{rigotti2013importance,
	title        = {The importance of mixed selectivity in complex cognitive tasks},
	author       = {Rigotti, Mattia and Barak, Omri and Warden, Melissa R and Wang, Xiao-Jing and Daw, Nathaniel D and Miller, Earl K and Fusi, Stefano},
	year         = 2013,
	journal      = {Nature},
	publisher    = {Nature Publishing Group UK London},
	volume       = 497,
	number       = 7451,
	pages        = {585--590}
}

@article{garg2026many,
  title={How Many Features Can a Language Model Store Under the Linear Representation Hypothesis?},
  author={Garg, Nikhil and Kleinberg, Jon and Peng, Kenny},
  journal={arXiv preprint arXiv:2602.11246},
  year={2026}
}

@article{adler2024complexity,
  title={On the complexity of neural computation in superposition},
  author={Adler, Micah and Shavit, Nir},
  journal={arXiv preprint arXiv:2409.15318},
  year={2024}
}

@article{hanni2024mathematical,
  title={Mathematical models of computation in superposition},
  author={H{\"a}nni, Kaarel and Mendel, Jake and Vaintrob, Dmitry and Chan, Lawrence},
  journal={arXiv preprint arXiv:2408.05451},
  year={2024}
}

@article{longon2025superposition,
  title={Superposition disentanglement of neural representations reveals hidden alignment},
  author={Longon, Andr{\'e} and Klindt, David and Khosla, Meenakshi},
  journal={arXiv preprint arXiv:2510.03186},
  year={2025}
}

@misc{joseph2025prismaopensourcetoolkit,
      title={Prisma: An Open Source Toolkit for Mechanistic Interpretability in Vision and Video}, 
      author={Sonia Joseph and Praneet Suresh and Lorenz Hufe and Edward Stevinson and Robert Graham and Yash Vadi and Danilo Bzdok and Sebastian Lapuschkin and Lee Sharkey and Blake Aaron Richards},
      year={2025},
      eprint={2504.19475},
      archivePrefix={arXiv},
      primaryClass={cs.CV},
      url={https://arxiv.org/abs/2504.19475}, 
}

@article{ILSVRC15,
Author = {Olga Russakovsky and Jia Deng and Hao Su and Jonathan Krause and Sanjeev Satheesh and Sean Ma and Zhiheng Huang and Andrej Karpathy and Aditya Khosla and Michael Bernstein and Alexander C. Berg and Li Fei-Fei},
Title = {{ImageNet Large Scale Visual Recognition Challenge}},
Year = {2015},
journal   = {International Journal of Computer Vision (IJCV)},
doi = {10.1007/s11263-015-0816-y},
volume={115},
number={3},
pages={211-252}
}

@article{hubel1962receptive,
  title={Receptive fields, binocular interaction and functional architecture in the cat's visual cortex},
  author={Hubel, David H and Wiesel, Torsten N},
  journal={The Journal of Physiology},
  volume={160},
  number={1},
  pages={106--154},
  year={1962},
  publisher={Wiley}
}

@article{hubel1968receptive,
  title={Receptive fields and functional architecture of monkey striate cortex},
  author={Hubel, David H and Wiesel, Torsten N},
  journal={The Journal of Physiology},
  volume={195},
  number={1},
  pages={215--243},
  year={1968},
  publisher={Wiley}
}

@article{cadieu2014deep,
  title={Deep neural networks rival the representation of primate {IT} cortex for core visual object recognition},
  author={Cadieu, Charles F and Hong, Ha and Yamins, Daniel LK and Pinto, Nicolas and Ardila, Diego and Solomon, Ethan A and Majaj, Najib J and DiCarlo, James J},
  journal={PLOS Computational Biology},
  volume={10},
  number={12},
  pages={e1003963},
  year={2014},
  publisher={Public Library of Science}
}

@article{quiroga2005invariant,
  title={Invariant visual representation by single neurons in the human brain},
  author={Quian Quiroga, Rodrigo and Reddy, Leila and Kreiman, Gabriel and Koch, Christof and Fried, Itzhak},
  journal={Nature},
  volume={435},
  number={7045},
  pages={1102--1107},
  year={2005},
  publisher={Nature Publishing Group}
}

@article{kanwisher1997fusiform,
  title={The fusiform face area: a module in human extrastriate cortex specialized for face perception},
  author={Kanwisher, Nancy and McDermott, Josh and Chun, Marvin M},
  journal={Journal of Neuroscience},
  volume={17},
  number={11},
  pages={4302--4311},
  year={1997},
  publisher={Society for Neuroscience}
}
\bibliographystyle{unsrt}  

\appendix
\newpage

\section{Proof of Theorem~\ref{thm:rsa-correlation} (RSA)}
\label{app:proof_rsa}

Our derivation relies on the following standard assumptions about the distribution of the latent variable vectors $z_i$:
\begin{enumerate}
    \item The latent vectors $z_1, \dots, z_d$ are i.i.d.
    \item The distribution has zero mean: $\mathbb{E}[z_i] = \mathbf{0}$.
    \item The distribution is white: $\mathbb{E}[z_i z_j\T] = \delta_{ij} I_n$.
\end{enumerate}
An immediate consequence is that for a large number of i.i.d.\ samples,
\begin{equation}
\frac{1}{d} Z Z\T = \frac{1}{d} \sum_{i=1}^d z_i z_i\T \to \mathbb{E}[z z\T] = I_n \quad \implies \quad Z Z\T \approx d I_n.
\end{equation}

\paragraph{RSMs in terms of Gram matrices.}
\begin{align}
    M(Y_a) &= (A_a Z)\T (A_a Z) = Z\T G_a Z, \\
    M(Y_b) &= (A_b Z)\T (A_b Z) = Z\T G_b Z,
\end{align}
so $M(Y_a)_{ij} = z_i\T G_a z_j$.

\paragraph{Mean of off-diagonal RSM elements.}
The empirical mean of all RSM elements in the asymptotic limit is
\begin{align}
    \mu_Y &\equiv \frac{1}{d^2} \sum_{i,j} M(Y)_{ij} = \frac{1}{d^2} \sum_{i,j} z_i\T G z_j \\
    &= \frac{1}{d} \sum_i z_i\T G \left[ \frac{1}{d} \sum_j z_j \right] \approx \frac{1}{d} \sum_i z_i\T G \,\mathbb{E}[z_j] = 0.
\end{align}
The off-diagonal-upper-triangular mean $\mu^{\mathrm{UT}}_Y$ differs from $\mu_Y$ only by $\mathcal{O}(1/d)$ correction and converges to 0.

\paragraph{Covariance and variance.}
\begin{align}
    \mathrm{Cov}(r_a, r_b) &\approx \frac{1}{d(d-1)} \sum_{i,j} (z_i\T G_a z_j)(z_j\T G_b\T z_i) \\
    &= \frac{1}{d-1} \sum_i z_i\T G_a \left[ \frac{1}{d} \sum_j z_j z_j\T \right] G_b z_i \\
    &\approx \operatorname{Tr}(G_a G_b).
\end{align}
Setting $G_a = G_b$ gives $\mathrm{Var}(r_a) = \|G_a\|_F^2$. Substituting into the Pearson formula:
\begin{equation}\label{eq:rsa-correlation}
    \boxed{\rho(Y_a, Y_b) \approx \frac{\operatorname{Tr}(G_a G_b)}{\sqrt{\|G_a\|_F^2 \|G_b\|_F^2}} = \frac{\langle G_a, G_b \rangle_F}{\|G_a\|_F \|G_b\|_F}.}
\end{equation}

\section{Proof of Theorem~\ref{thm:cka_lin} (Linear CKA)}
\label{app:proof_cka_lin}

\paragraph{Centering.} With $H_d = I_d - \frac{1}{d}\mathbf{1}\mathbf{1}\T$,
\begin{align}
Y H_d &= Y - \frac{1}{d} Y \mathbf{1}\mathbf{1}\T = Y - \bar y \,\mathbf{1}\T \\
&\approx Y - (A \,\mathbb{E}[z]) \mathbf{1}\T = Y.
\end{align}

\paragraph{HSIC computation.} With linear kernel $K_a = Y_a\T Y_a$, $K_b = Y_b\T Y_b$, and using $H_d^2 = H_d$,
\begin{align}
\operatorname{Tr}(K_a H_d K_b H_d) &= \operatorname{Tr}(Y_a\T Y_a Y_b\T Y_b) \\
&= \operatorname{Tr}(Z\T G_a Z Z\T G_b Z) \\
&= \operatorname{Tr}(Z Z\T G_a Z Z\T G_b) \\
&\approx d^2 \operatorname{Tr}(G_a G_b).
\end{align}
Similarly $\operatorname{Tr}(K_a H_d K_a H_d) \approx d^2 \operatorname{Tr}(G_a^2)$. Substituting:
\begin{equation}
\mathrm{CKA}_{\mathrm{Lin}}(Y_a, Y_b) \approx \frac{\operatorname{Tr}(G_a G_b)}{\sqrt{\operatorname{Tr}(G_a^2) \operatorname{Tr}(G_b^2)}},
\end{equation}
identical to the asymptotic RSA result.

\section{Proof of Theorem~\ref{thm:ols} (Linear Regression)}
\label{app:proof_linreg}

\paragraph{OLS estimator.} Standard OLS:
\begin{equation}
\hat W = Y_b Y_a\T (Y_a Y_a\T)^{-1}.
\end{equation}
Substituting $Y_a = A_a Z$, $Y_b = A_b Z$ and $Z Z\T \approx d I_n$:
\begin{align}
Y_b Y_a\T &\approx d(A_b A_a\T), \\
Y_a Y_a\T &\approx d(A_a A_a\T), \\
\hat W &\approx A_b A_a\T (A_a A_a\T)^{-1}.
\end{align}

\paragraph{Mean-squared error.} Prediction error $E = Y_b - \hat W Y_a = (A_b - \hat W A_a) Z$, so
\begin{align}
\|E\|_F^2 &= \operatorname{Tr}((A_b - \hat W A_a)\T (A_b - \hat W A_a) Z Z\T) \\
&\approx d \|A_b - \hat W A_a\|_F^2.
\end{align}
Dividing by the total number of predicted entries $m_b d$:
\begin{equation}
\mathrm{MSE}(Y_b | Y_a) \approx \frac{1}{m_b} \|A_b - \hat W A_a\|_F^2.
\end{equation}

\paragraph{Explained variance.} With $\bar y_b \approx 0$, $\mathrm{SS}_{\mathrm{tot}} \approx d\,\operatorname{Tr}(A_b\T A_b)$ and $\mathrm{SS}_{\mathrm{res}} \approx d\,\operatorname{Tr}((A_b - \hat W A_a)\T (A_b - \hat W A_a))$:
\begin{equation}
R^2 = 1 - \frac{\operatorname{Tr}((A_b - \hat W A_a)\T (A_b - \hat W A_a))}{\operatorname{Tr}(A_b\T A_b)}.
\end{equation}

\paragraph{Pearson correlation.} With indices $i,j$ corresponding to system dimensions,
\begin{align}
\mathrm{Cov}(\hat Y_b^i, Y_b^j) &\approx (\hat W A_a A_b\T)_{ij}, \\
\mathrm{Var}(\hat Y_b^i) &\approx (\hat W A_a A_b\T)_{ii}, \\
\mathrm{Var}(Y_b^j) &\approx (A_b A_b\T)_{jj}.
\end{align}
Therefore
\begin{equation}
\rho(\hat Y_b, Y_b)_{ij} \approx \frac{(\hat W A_a A_b\T)_{ij}}{\sqrt{(\hat W A_a A_b\T)_{ii} (A_b A_b\T)_{jj}}}.
\end{equation}

\section{Gaussian-Projection Simulations}
\label{app:gaussian_sims}
\subsection{Full feature overlap}
This appendix validates the theory of Section~\ref{sec:alignment} using purely random Gaussian projections (i.e., the idealized analytic setup), independent of any SAE machinery. These experiments are the ones reported in the current draft of the paper; we retain them here as direct empirical confirmation of the closed-form theorems.

\paragraph{Simulation setup.}
We fix the latent space dimension at $n = 1000$ and set both system dimensions equal, $m_a = m_b \equiv m$. Latent feature vectors are sampled as $Z \in \mathbb{R}^{n \times d}$ with $d = 16384$ samples, where each entry is drawn uniformly, $Z_{i,j} \sim \mathcal{U}(0,1)$. To enforce sparsity, we retain only the $k$ largest activations within each sample and set the remainder to zero, yielding a dataset in which each latent vector has at most $k$ non-zero entries.

The two projection matrices $A_a, A_b \in \mathbb{R}^{m \times n}$ are drawn independently with entries $A_{i,j} \sim \mathcal{N}(0,1)$, producing two distinct linear compressions of the same latent content. We vary the degree of superposition by sweeping $m$ over the range $[0.2\,k\ln(n/k),\; 5\,k\ln(n/k)]$. For each value of $m$, we repeat this procedure over 200 independent draws of the projection matrix pair $(A_a, A_b)$ and report the mean alignment score across draws as a function of $m$. All experiments are repeated across several values of $k$  (2, 4, 8, 16, 32, and 64) to assess the role of sparsity, and alignment is measured using RSA, linear regression ($R^2$), and CKA with a linear kernel.

Throughout, we annotate results relative to the scaling predicted by compressed sensing theory. For a Gaussian random projection matrix, the restricted isometry property holds with high probability when the number of measurements satisfies $m = \mathcal{O}(k \ln(n/k))$ \citep{donoho_compressed_2006, candes2006stable}. The exact constant is not known in general, so we use
\begin{equation}
    m_{\text{cs}} = k \ln\!\frac{n}{k}
\end{equation}
as a natural unit for the system dimension, distinguishing the regime in which latent features are in principle recoverable (CS; $m \geq m_{\text{cs}}$) from the regime in which they are not (No CS; $m < m_{\text{cs}}$). Note that the true recovery threshold may differ from $m_{\text{cs}}$ by a constant factor.

\subsection{Partial feature overlap}
We define the feature overlap ratio $u$ as the geometric-mean-normalized count of shared features:
\begin{equation}
u \equiv \frac{l_{ab}}{\sqrt{l_a l_b}},
\end{equation}
where $l_a$, $l_b$, and $l_{ab}$ denote the number of features in system $A$, system $B$, and the number of features shared between them, respectively. When $u = 1$, the two systems share all of their features (full overlap); when $u < 1$, each system retains features unique to itself in addition to the shared subset.

We fix $l_a = l_b \equiv l = 1000$ and vary $u$ from $0.7$ to $1.0$, so that the number of shared features ranges from $l_{ab} = 0.7 \cdot l$ to $l_{ab} = l$. For each value of $u$, the total latent dimensionality is $n = l_a + l_b - l_{ab} = 2l - u \cdot l$, which grows as overlap decreases because the two systems collectively span more distinct features.

Latent features are sampled as $Z \in \mathbb{R}^{n \times d}$ with $d = 8192$ and entries $Z_{i,j} \sim \mathcal{U}(0,1)$, then sparsified by retaining only the top $n \cdot k/l$ activations per sample. To implement the desired feature overlap structure, projection matrices $A_a \in \mathbb{R}^{m_a \times n}$ and $A_b \in \mathbb{R}^{m_b \times n}$ are drawn with entries $A_{i,j} \sim \mathcal{N}(0,1)$ and then column-masked: $A_a$ is constrained to have non-zero columns only in the $l_a$ dimensions corresponding to system $A$'s features, and $A_b$ likewise for its $l_b$ dimensions, with exactly $l_{ab}$ columns active in both matrices. This ensures that each system projects only its own features into neural activity, while the shared features are encoded by both. 

We ask whether superposition can distort alignment even when the two systems differ in size. To this end, we fix $m_a = 3 \cdot k\ln(l/k)$, placing system $A$ comfortably in the CS regime, and vary $m_b$ over $[ k\ln(l/k), 3 \cdot k\ln(l/k) ]$, so that system $B$ ranges from heavily compressed to equally uncompressed. This procedure is repeated across multiple overlap ratios $u$ (0.7 - 1.0) and sparsity fractions $k/l$ (0.05, 0.1, 0.15 and 0.2). Since our analytical results in Section~\ref{sec:theory} show that RSA and linear CKA converge to the same closed-form expression in the asymptotic limit (Theorems~\ref{thm:rsa-correlation} and~\ref{thm:cka_lin}), and since the full overlap simulations confirm that all three metrics exhibit similar trends as a function of superposition compression (Figure~\ref{fig:gauss_full}), we use CKA with a linear kernel as a representative metric for the partial overlap experiments, without loss of generality.

\section{Supervised TopK SAE: Training Details}
\label{app:sup_sae_details}

We generate set of latents $Z \in \mathbb{R}^{n \times d}$, where $n=1024$, through random uniform sampling, then zero out all but the top k-activating entries in every column to fulfill the sparsity condition. We choose $k=32$ and $k=64$ respectively. The TopK SAE architectures are as follows:

\begin{align*}
    \hat{z} & = \text{TopK}\left( \text{ReLU}\left(E \cdot y \right) + b \right) \\
    y & = D \cdot z
\end{align*}
Where $E \in \mathbb{R}^{n \times m}$, $D \in \mathbb{R}^{m \times n}$, and $b\in \mathbb{R}^n$ are the encoder matrix, decoder matrix, and encoder bias respectively.
We train the supervised SAE by minimizing the following objective:
\begin{align*}
    L = \mathbb{E}\left[ |z - \hat{z}|^2\right]
\end{align*}

We randomly initialize $D$, $E$ and $b$ through standard normal distribution, then optimize with Adam optimizer with initial learning rate of 1e-3. We use a training set of length $d = 8192$, and feed the whole dataset as a single batch during training, which continues for 2e5 training steps.

We sample 50 random seeds and report the mean results for alignment measurement.

\section{Unsupervised TopK SAE: Training Details}
\label{app:unsup_sae_details}
Analogous to \ref{app:sup_sae_details}, we train unsupervised SAE $y$ by minimizing the following loss function:
\begin{align*}
L & = \mathbb{E}\left[ |y - \hat{y}|^2\right] \\
\hat{y} & = D \cdot \hat{z} \\
\hat{z} & = \text{TopK}\left( \text{ReLU}\left(E \cdot y \right) + b \right)
\end{align*}
We set dictionary size $n' = n = 1024$ and sparsity matched to true $k$ (32/64). We optimize with Adam optimizer with initial learning rate of 1e-3. We use a training set of length $d = 8192$, and feed the whole dataset as a single batch during training, which continues for 1e6 training steps.

\section{Pretrained-SAE Experiments: Details}
\label{app:real_sae_details}

We use the available DINO and CLIP models on Prisma repository \citep{joseph2025prismaopensourcetoolkit}, which consist of sub layers resid-post and mlp-out in CLIP and sub layers resid-post in DINO, giving 630 unique model layer- model layer pairs and 1260 data point. We obtain neural activations of each sublayer on the ImageNet1K validation dataset \citep{ILSVRC15}, which consists of 50000 natural images. For each pair of model layer $(a,b)$ comparison, we train TopK SAEs setting $k=64$ and latent dimension $n=2m$, $m$ again being the number of neurons in the original system. The TopK SAEs are trained jointly, constrained to a shared latent space through minimizing the loss function
\begin{equation}
\mathcal{L} = \sum_{i,j \in \{a,b\}} \mathrm{MSE}(y_i, \hat{y}_{ij}), \qquad \hat{y}_{ij} = D_i \cdot \text{TopK}_k(E_j \, y_j),
\end{equation}

We pick the SAE latent dimension to be $n=2 \times m$. We initialize the SAE encoder and decoder weights with a standard normal distribution. We perform a train test split of test size $0.2$. i.e. on the ImageNet1K validation, train set contains 40000 stimuli and test contains 10000. We use AdamW optimizer with a learning rate of 1e-3, and feed the whole training set at each gradient step. We train the TopK SAEs for 2e6 epochs.

We report the alignment (pearson r between OLS predictions and targets) on the test set.

\end{document}